\documentclass{article}
\usepackage{spconf,amsmath,graphicx}
\usepackage{amsmath,graphicx}
\usepackage{amsthm}
\usepackage{cite}
\usepackage{graphicx}
\usepackage{amssymb}
\usepackage{amsfonts}
\usepackage{multirow}
\usepackage{mathrsfs}
\usepackage{color}
\usepackage{algorithm}
\usepackage{algorithmic}
\usepackage{multirow}
\usepackage{amsmath}
\usepackage{graphics}
\usepackage{graphicx}
\usepackage{epsfig}
\usepackage{cases}
\usepackage{amsmath,bm}
\usepackage{wasysym}
\usepackage{graphicx}
\usepackage{subfigure}
\usepackage{hyperref}
\hypersetup{ hidelinks = true, }

\newtheorem{mytheorem}{Theorem}
\newtheorem{mylemma}{Lemma}
\newtheorem{mydef}{Definition}

\newtheorem{myproperty}{Property}

\title{CONVERGENCE ANALYSIS OF THE \\INFORMATION MATRIX in GAUSSIAN BELIEF  PROPAGATION}

\name{Jian Du$^{\dagger}$\thanks{This work is partially supported by NSF grant \# CCF1513936.},  Shaodan Ma$^{\star}$,  Yik-Chung Wu$^{\ddag}$,
 Soummya Kar$^{\dagger}$ and  Jos{\'e} M. F. Moura$^{\dagger}$ }
\address{Electrical and Computer Engineering, Carnegie Mellon University$^{\dagger}$, Pittsburgh, PA \\
    Electrical and Computer Engineering, University of Macau$^{\star}$, Macau\\
    Electrical and Electronic Engineering, The University of Hong Kong$^{\ddag}$, Hong Kong}
\begin{document}
\newcommand*{\QEDA}{\hfill\ensuremath{\blacksquare}}
\def\N{{\mathcal{N}}}
\def\B{{\mathcal{B}}}
\def\I{{\textbf{I}}}
\def\diag{{\textrm{diag}}}
\def\i {{ -i}}

%
\maketitle
\begin{abstract}
Gaussian belief propagation (BP) has been widely used for distributed
estimation in large-scale networks such as the
smart grid,  communication networks, and social networks, where local measurements/observations are scattered over a wide geographical area.
However, the convergence of Gaussian BP is still an open issue.
In this paper, we consider the convergence  of Gaussian BP, focusing in particular on the convergence of the information matrix.
We show analytically that
 the exchanged message information matrix  converges  for arbitrary positive semidefinite initial value, and its distance to the unique  positive definite limit matrix decreases exponentially fast.
\end{abstract}
\section{Introduction}\label{Section 1}
In large-scale linear parameter estimation with Gaussian measurements,
Gaussian Belief Propagation (BP)
 \cite{DiagnalDominant} provides an efficient distributed way to compute the marginal distribution of the unknown variables, and it has been adopted in a variety of topics ranging from  distributed power state estimation \cite{A1} in smart grid, distributed beamforming \cite{A2} and synchronization  \cite{JianClock,cfo} in wireless communication networks,  fast solver for system of linear equations \cite{A4}, distributed rate control in ad-hoc networks \cite{A5},  factor analyzer network \cite{A6}, sparse Bayesian learning \cite{A7}, 
 to peer-to-peer rating in social networks \cite{A9}.
 It has been shown that Gaussian BP computes the optimal centralized estimator if it converges \cite{MRFtoFG}.

Although with great empirical success, the  major challenge that hinders Gaussian BP to realize its full potential is the lack of theoretical guarantees of convergence in loopy networks.
Sufficient convergence conditions for Gaussian BP have been developed in \cite{DiagnalDominant,WalkSum1,minsum09,Suqinliang}  when the underlying Gaussian distribution is expressed in terms of pairwise connections between scalar variables (also known as Markov random field (MRF)).
These works focus on the convergence analysis of Gaussian BP for computing the marginal distribution of a joint distribution with pairwise factors.
However,  the iterative equations for Gaussian BP on MRFs are different from that for distributed estimation problems such as in \cite{A1,A2,JianClock,cfo,A4,A5,A6,A7,A8},  where high order factors (non-pairwise) and vector-valued variables  are involved.
Therefore, these existing conditions and analysis methods are not applicable to distributed estimation problems.
In this paper, we study the convergence analysis of Gaussian BP for distributed parameter estimation focusing on the convergence of message information matrix.
We show analytically that,
with arbitrary positive
semidefinite  matrix  initialization,
the  message information matrix being exchanged among nodes  converges  and its distance to the unique  positive definite  limit matrix decreases exponentially.

Note that distributed estimation based on the consensus$+$\\innovations philosophy   proposed in \cite{Kar-SPS, Kar-SIAM} (see also the related family of diffusion algorithms~\cite{Sayed10Diffusion}) converges to the optimal centralized estimator under the assumption of global observability
of the (aggregate) sensing model and connectivity of the inter-agent communication network.
In particular, these algorithms allow the communication or message exchange network to be different from the physical coupling network and the former could be arbitrary with cycles (as long as it is connected).
The results in \cite{Kar-SPS, Kar-SIAM} imply that the unknown variables $\textbf x$  can be reconstructed completely at each node in the network.
For large-scale networks with high dimensional $\textbf x$, it may be impractical to reconstruct $\textbf x$ at every node.
In \cite[section 3.4]{Kar-thesis}, the author developed approaches to address this problem, where each node can reconstruct a set of unknown variables that should be larger than the set of variables that influence its local measurement.
This paper studies a different distributed estimation problem when each node estimates only its own unknown variables under pairwise independence condition of the unknown variables; this leads to lower dimensional data exchanges between neighbors.

\section{Computation Model}\label{hybrid}
Consider a general connected network
of $M$  nodes, with $\mathcal{V}=\{1,\ldots, M\}$ denoting the set of nodes,
and $\mathcal{E}_{\textrm{Net}} \subset \mathcal{V} \times  \mathcal{V}$ as the set of all undirect communication links in the network, i.e., if $i$ and $j$ are within the communication range, $(i, j) \in \mathcal{E}_{\textrm{Net}}$.
At every node $n \in \mathcal{V}$,
the local observations are in the form of
$\textbf{y}_n = \sum_{i\in  n\cup\mathcal{I}(n)}
\textbf{A}_{n,i}\textbf{x}_i + \textbf{z}_n,
$
where
$\mathcal{I}(n)$ denotes the set of direct neighbors of node $n$ (i.e., all nodes $i$ with $(n,i) \in \mathcal{E}_{\textrm{Net}}$),
$\textbf{A}_{n,i}$ is a known coefficient matrix with full column rank,
$\textbf{x}_i$ is the local unknown parameter at node $i$ with  dimension $N_i \times 1$, and with the prior distribution $p(\textbf{x}_i)\sim \mathcal{N}(\textbf{x}_i|\textbf{0},\textbf{W}_{i})$,
and $\textbf{z}_n$ is the additive noise with distribution $\textbf{z}_n\sim \mathcal{N}(\textbf{z}_n|\textbf{0},\textbf{R}_n)$.
It is assumed that
$p(\textbf{x}_i, \textbf{x}_j)=p(\textbf{x}_i)p(\textbf{x}_j)$
and
$p(\textbf{z}_i,\textbf{z}_j)
=p(\textbf{z}_i)p(\textbf{z}_j)$ for $i\neq j$.
The goal is to estimate $\textbf{x}_i$, based on $\textbf{y}_n$, $p(\textbf{x}_i)$ and $p(\textbf{z}_n)$.

The Gaussian
BP algorithm can be derived  over the corresponding factor graph to compute the estimate of $\textbf x_n$ for all $n\in \mathcal V$ \cite{journalversion}. It
involves two kinds of messages:
One is the message
from a variable node $\textbf x_j$ to its neighboring factor node $f_n$, defined as
\begin{equation} \label{BPv2f1}
m^{(\ell)}_{j \to f_n}(\textbf x_j)
= p(\textbf x_j)
\prod_{f_k\in \B(j)\setminus f_n}m^{(\ell-1)}_{f_k\to j}(\textbf x_j),
\end{equation}
where $\B(j)$ denotes the set of neighbouring factor nodes  of $\textbf x_j$,
and $m^{(\ell-1)}_{f_k\to j}(\textbf x_j)$ is the   message  from $f_k$ to $\textbf x_j$ at time $l-1$.
The second type of
message is from a factor node $f_n$ to a neighboring variable node $\textbf{x}_i$,    defined as
\begin{equation}\label{BPf2v1}
m^{(\ell)}_{f_n \to i}(\textbf{x}_i)
= \!\! \int\!\!\! \cdots \!\!\!\int
\!\!f_n \times\!\!\!\!\!\!\!
\prod_{j\in\B(f_n)\setminus i}
\!\!\!\!\!\!
m^{ (\ell)}_{j \to f_n}(\textbf x_j)
\,\mathrm{d}\{\textbf x_j\}_{j\in\B(f_n)\setminus i},
\end{equation}
where $\B(f_n)$ denotes the set of neighboring variable nodes of $f_n$.
The process iterates between equations (\ref{BPv2f1}) and (\ref{BPf2v1}).
At each iteration $l$, the approximate marginal distribution, also named belief, on $\textbf{x}_i$ is computed locally at $\textbf{x}_i$ as
\begin{equation} \label{BPbelief}
b_{\textrm{BP}}^{(\ell)}(\textbf{x}_i)
 = p(\textbf{x}_i) \prod_{ f_n\in \B(i)} m^{(\ell)}_{ f_n \to i}(\textbf{x}_i).
\end{equation}


It can be  shown \cite{journalversion} that the general expression for the message from variable node to factor node  is

\begin{equation} \label{BPvs2f1}
m^{(\ell)}_{j \to f_n}(\textbf x_j) \propto
\exp
\big\{-\frac{1}{2}
||\textbf x_j- \textbf{v}^{(\ell)}_{j\to f_n}||
_{\textbf{C}^{(\ell)}_{j\to f_n}}
\big\},
\end{equation}
where $\textbf{C}_{j\to f_n}^{(\ell)}$ and $ \textbf{v}_{j\to f_n}^{(\ell)}$ are the message covariance matrix and mean vector  received at variable node $j$ at the $l$-$\textrm{th}$ iteration, with
\begin{equation} \label{v2fV}
\big[\textbf{C}^{(\ell)}_{j \to f_n}\big]^{-1}
= \textbf{W}_j^{-1} +
\sum_{f_k\in\B(j)\setminus f_n}
\big[\textbf{C}_{f_k\to j}^{(\ell-1)}\big]^{-1}.
\end{equation}
Furthermore,
the message  from factor node to variable node is given by \cite{journalversion}
\begin{equation} \label{f2v}
m^{(\ell)}_{f_n \to i}(\textbf{x}_i)\propto
\exp
\big\{-\frac{1}{2}
||\textbf{x}_i- \textbf{v}^{(\ell)}_{f_n\to i}||
_{\textbf{C}^{(\ell)}_{f_n\to i}}
\big\},
\end{equation}
where $\textbf{C}_{f_k\to j}^{(\ell-1)}$ and $ \textbf{v}_{f_k\to j}^{(\ell-1)}$ are the message covariance matrix and mean vector  received at variable node $j$ at the $l-1$ iteration
with
\begin{equation}\label{Cov}
\begin{split}
[\textbf{C}^{(\ell)}_{f_n\to i} ]^{-1}
=\!
\textbf{A}_{n,i}^T
\big[ \textbf{R}_n
+\!\!\!\!\!\!
\sum_{j\in\B(f_n)\setminus i} \!\!\!\!\! \textbf{A}_{n,j}\textbf{C}^{(\ell)}_{j\to f_n}\textbf{A}_{n,j}^T \big]^{-1}\!
\textbf{A}_{n,i}.
\end{split}
\end{equation}

The following lemma shown in \cite{journalversion} indicates that   setting the initial message covariances $[\textbf{C}_{f_n\to i}^{(0)}]^{-1}\succeq \textbf{0}$  for all $n, i\in \mathcal{V}$
guarantees  $[\textbf{C}^{(\ell)}_{j\to f_n}]^{-1}\succ \textbf{0}$ for $l \geq 1$.

\begin{mylemma}\label{pdlemma}
Let the initial messages at factor node $f_k$ be in Gaussian function forms with  covariance $[\textbf{C}_{f_k\to j}^{(0)}]^{-1} \succeq \textbf{0}$ for all $k \in \mathcal{V}$ and $j \in \mathcal{B}(f_k)$.
Then
$[\textbf{C}_{j\to f_n}^{(\ell)}]^{-1}\succ \textbf{0}$
 and $[\textbf{C}^{(\ell)}_{f_k \to j}]^{-1} \succ \textbf{0}$
 for all $l\geq 1$ with $j \in \mathcal{V}$ and
$f_n, f_k \in \mathcal{B}(j)$.
 Furthermore,
in this case,
all the messages $m^{(\ell)}_{j \to f_n}(\textbf x_j)$ and
$m^{(\ell)}_{f_k \to j}(\textbf{x}_i)$ exist and are in Gaussian form.
\end{mylemma}

For this factor graph based approach, according to the message updating procedure  (\ref{BPvs2f1}) and (\ref{f2v}),  message exchange is only needed between neighboring nodes.
For example, the messages transmitted from node $n$ to its neighboring node $i$ are   $m_{f_n\to i}^{(\ell)}(\textbf{x}_i)$ and $m_{n\to f_i}^{(\ell)}(\textbf x_n)$.
Thus, the message passing scheme given in (\ref{BPv2f1}) and (\ref{BPf2v1}) automatically conforms with the network topology.
Furthermore, if the messages $m_{j\to f_n}^{(\ell)}(\textbf x_j)$
and $m_{f_n\to i}^{(\ell)}(\textbf{x}_i)$ exist for all $l$
(which can be achieved using Lemma \ref{pdlemma}),
the messages  are Gaussian, therefore only the corresponding mean vectors and information matrices (inverse of covariance matrices) are needed to be exchanged.

Finally, if the BP messages exist, according to the definition of belief in (\ref{BPbelief}),
$b_{\textrm{BP}}^{(\ell)}(\textbf{x}_i)$ at  iteration  $l$ is computed as \cite{journalversion}
\begin{equation} \label{belief2}
\mathbf{b}_{\textrm{BP}}^{(\ell)}(\textbf{x}_i)
=p(\textbf{x}_i)
\!\!\prod_{f_n\in\mathcal B(i)} m_{f_n\to i}^{(\ell)}(\textbf{x}_i)
\propto  \mathcal{N}\big(\textbf{x}_i|
\boldsymbol{\mu}_i^{(\ell)}, \textbf{P}_i^{(\ell)}\big),
\end{equation}
with $
\textbf{P}_i^{(\ell)} =
\big[\textbf{W}_i^{-1}
+\sum_{f_n\in\B(i)} \big[
\textbf{C}_{f_n\to i}^{(\ell)}\big]^{-1}\big]^{-1},
$
and $
\boldsymbol{\mu}_i^{(\ell)}=\textbf{P}_i^{(\ell)}\big[
\sum_{f_n\in\B(i)}
\big[\textbf{C}_{f_n\to i}^{(\ell)}\big]^{-1}\textbf{v}^{(\ell)}_{f_n\to i}\big].
$
The iterative computation terminates when message (\ref{BPvs2f1}) or message (\ref{f2v}) converges to a fixed value or the maximum number of iterations is reached.

\section{Convergence of  Information Matrices}\label{analysis}
The challenge of deploying the BP algorithm for large-scale networks is determining whether it will converge.
In particular, it
is generally known that if the factor graph contains cycles, the BP algorithm may
diverge.
Thus, determining  convergence conditions for the BP algorithm is very important.
Sufficient conditions for the convergence of Gaussian BP with scalar variable in loopy graphs are available in \cite{DiagnalDominant, WalkSum1,Suqinliang}.
 However,  they are derived based on  pairwise graphs with local functions that only involve two variables.
This is in sharp contrast to the model considered in this paper, where the
$f_n$   involves high-order interactions between vector variables, and
thus the convergence results in \cite{DiagnalDominant, WalkSum1, Suqinliang} cannot be applied to the factor graph based vector-form Gaussian BP.

Due to the recursively updating property of  $m_{j\to f_n}^{(\ell)}(\textbf x_j)$ and $m_{f_n\to i}^{(\ell)}(\textbf{x}_i)$ in (\ref{BPvs2f1}) and (\ref{f2v}), the message evolution can be simplified by combining these two kinds of messages into one.
By substituting
 $ \big[\textbf{C}^{(\ell)}_{j \to f_n}\big]^{-1}$ in (\ref{v2fV}) into   (\ref{Cov}), the updating of the message covariance matrix inverse, named message information matrix in the following, can be denoted as
\begin{eqnarray}\label{CovFunc}
\begin{split}
\big[\textbf{C}_{f_n\to i}^{(\ell)}\big]^{-1}
=&
\textbf{A}_{n,i}^T \big[\textbf{R}_n
+ \sum_{j\in\B(f_n)\setminus i} \textbf{A}_{n,j}
\big[
 \textbf{W}_{j}^{-1} \\
&+
\sum_{f_k\in\B(j)\setminus f_n}
\big[\textbf{C}_{f_k\to j}^{(\ell-1)}\big]^{-1}
\big]^{-1}
\textbf{A}_{n,j}^T \big]^{-1}
\textbf{A}_{n,i}\nonumber\\
\triangleq &
\mathcal{F}_{n\to i}
\big(\{
\big[\textbf{C}_{f_k\to j}^{(\ell-1)}\big]^{-1}\}_{(f_k, j)\in \mathcal{\widetilde{B}}(f_n, i)}
  \big),
\end{split}
\end{eqnarray}
where $\mathcal{\widetilde{B}}(f_n, i)=\{(f_k, j) | j \in \B(f_n)\setminus i,  f_k\in \B(j)\setminus f_n\}$.
Observing that $\textbf{C}_{f_n\to i}^{(\ell)}$ in (\ref{CovFunc}) is independent of $\textbf{v}^{(\ell)}_{j\to f_n} $ and $\textbf{v}^{(\ell)}_{f_n\to i}$ in (\ref{BPvs2f1}) and (\ref{f2v}), we can  focus on the convergence property of $[\textbf{C}_{f_n\to i}^{(\ell)}]^{-1}$ alone.

To consider the updates of all message
information matrices, we {{introduce} the following definitions.
Let
${\textbf{C}}^{(\ell-1)}
\triangleq
\texttt{Bdiag}
(\{[\textbf{C}_{f_n\to i}^{(\ell-1)}]^{-1}\}_{n\in \mathcal{V},i\in \B(f_n)})$ be
a block diagonal  matrix with diagonal blocks
being the   message information matrices in the network at time $l-1$
with index arranged in ascending order first on $n$ and then on $i$.
Using the definition of $\textbf{C}^{(\ell-1)}$, the term $\sum_{f_k \in \mathcal B(j) \backslash f_n} [\textbf{C}_{f_k\rightarrow j}^{(\ell-1)} ]^{-1}$ in (\ref{CovFunc}) can be written as $\boldsymbol{\Xi}_{n,j} \textbf{C}^{(\ell-1)} \boldsymbol{\Xi}_{n,j}^T$, where $\boldsymbol{\Xi}_{n,j}$ is for selecting appropriate components from $\textbf{C}^{(\ell-1)}$ to form the summation.
Further, define $\textbf{H}_{n,i}=[\{ \textbf{A}_{n,j} \}_{j\in \mathcal B(f_n) \backslash i}]$,
$\boldsymbol{\Psi}_{n,i}  = \texttt{Bdiag} (\!\{\textbf{W}_j ^{-1} \}_{j\in \mathcal B(f_n) \backslash i}\!)$ and $\textbf{K}_{n,i}\!\!=\!\!\texttt{Bdiag} (\!\{ \boldsymbol{\Xi}_{n,j} \}_{j\in \mathcal B(f_n) \backslash i}\!) $, all with component blocks arranged with ascending order on $j$.  Then (\ref{CovFunc}) can be written as
\begin{equation}\label{CovFunc3}
\begin{split}
    [\textbf{C}^{(\ell)}_{f_n\rightarrow i}]^{-1}
    =&\textbf{A}_{n,i}^T\big \{\textbf{R}_n+ \textbf{H}_{n,i}[\boldsymbol \Psi_{n,i}
    + \textbf{K}_{n,i} (\textbf{I}_{|\mathcal{B}(f_n)|-1} \\ &\otimes \textbf{C}^{(\ell-1)}) \textbf{K}_{n,i}^T ]^{-1} \textbf{H}_{n,i}^T \big\} ^{-1}\textbf{A}_{n,i}.
\end{split}
\end{equation}

Now, we define the function $\mathcal{F}\triangleq\{\mathcal{F}_{1\to k}, \ldots, \mathcal{F}_{n\to i}, \ldots,\\ \mathcal{F}_{n \to M}\}$ that satisfies
${\textbf{C}}^{(\ell)} = \mathcal{F}({\textbf{C}}^{(\ell-1)}) $.
Then, by stacking $\big[\textbf{C}_{f_n\to i}^{(\ell)}\big]^{-1}$ on the left side of
(\ref{CovFunc3}) for all $n$ and $i$ as the block diagonal matrix $\textbf{C}^{(\ell)}$, we obtain
\begin{eqnarray}\label{CovFunc5}
  \textbf{C}^{(\ell)}
  &=& \textbf{A}^T \big \{ \boldsymbol{\Omega}+ \textbf{H}[\boldsymbol{\Psi} + \textbf{K} (\mathbf{I}_\varphi \otimes \textbf{C}^{(\ell-1)}) \textbf{K}^T ]^{-1} \textbf{H}^T \big\} ^{-1}\textbf{A}, \nonumber\\
   &\triangleq& \mathcal F(\textbf{C}^{(\ell-1)}),
\end{eqnarray}
where $\textbf{A}$, $\textbf{H}$,
$\boldsymbol{\Psi}$,  and $\textbf{K}$ are block diagonal matrices with block elements $\textbf{A}_{n,i}$, $\textbf{H}_{n,i}$, $\boldsymbol{\Psi}_{n,i} $, and $\textbf{K}_{n,i}$, respectively, arranged in ascending order, first on $n$ and then on $i$ (i.e., the same order as $[\textbf{C}^{(\ell)}_{f_n \rightarrow i}]^{-1}$ in $\textbf{C}^{(\ell)}$).
Furthermore,
$\varphi={\sum _{n=1} ^M |\mathcal B(f_n)|(|\mathcal B(f_n)|-1)}$
and
$\boldsymbol{\Omega}$ is a block diagonal matrix with diagonal blocks $\textbf{I} _{|B(f_n)|} \otimes \textbf{R}_n$ with ascending order on $n$.
We first present  properties of the updating operator $\mathcal{F}(\cdot)$,
with
the  proof   given in \cite{journalversion}.

\begin{myproperty} \label{P_FUN}
 The updating operator $\mathcal{F}(\cdot)$ satisfies the following properties:
\end{myproperty}

\noindent P \ref{P_FUN}.1:
$\mathcal{F}(\textbf{C}^{(\ell)}) \succeq \mathcal{F}(\textbf{C}^{(\ell-1)})$, if $\textbf{C}^{(\ell)} \succeq \textbf{C}^{(\ell-1)}\succeq \textbf{0}$.

\noindent P \ref{P_FUN}.2: $\alpha\mathcal{F}(\textbf{C}^{(\ell)}) \succ  \mathcal{F}(\alpha \textbf{C}^{(\ell)})$
and
$\mathcal{F}(\alpha^{-1}\textbf{C}^{(\ell)}) \succ  \alpha^{-1}\mathcal{F}(\textbf{C}^{(\ell)})$, if $\textbf{C}^{(\ell)} \succ \textbf{0}$ and $\alpha>1$.

\noindent P \ref{P_FUN}.3:
Define
$\textbf{U}\triangleq \textbf{A}^T  \boldsymbol{\Omega}^{-1}\textbf{A}$
and $\textbf{L}\triangleq
\textbf{A}^T \Big [  \boldsymbol{\Omega}+ \textbf{H}\boldsymbol{\Psi}^{-1} \textbf{H}^T \Big] ^{-1}\!\!\textbf{A}$.
With arbitrary $\textbf{C}^{(0)}\succeq \textbf{0}$,
$\mathcal{F}(\textbf{C}^{(\ell)})$ is bounded by
$\textbf{U} \succeq  \mathcal{F}(\textbf{C}^{(\ell)})\succeq \textbf{L}\succ \textbf{0}$ for $l\geq 1$.

In this paper, $\textbf{X} \succeq\textbf{Y}$ ($\textbf{X} \succ \textbf{Y}$) means that $\textbf{X} - \textbf{Y}$ is positive semidefinite (definite).
Based on the above properties of $\mathcal{F}(\cdot)$, we can establish the convergence property
for the information matrices.
The following theorem establishes that
there exists a unique fixed point for the mapping
$\mathcal{F}(\cdot)$.
The proof is omitted due to space restrictions; it is provided in \cite{journalversion}.
\begin{mytheorem} \label{unique}
With  $\textbf{C}^{(0)}\succeq\textbf{0}$, there exists a unique   positive definite fixed point for the mapping $\mathcal F(\cdot)$.
\end{mytheorem}

Lemma \ref{pdlemma} states that with arbitrary positive
semidefinite (p.s.d.) initial message information matrices, the message
information matrices will be kept as positive
definite (p.d.) at every iteration.
On the other hand, Theorem \ref{unique} indicates that there exists a unique fixed point for the mapping $\mathcal F$.
Next, we will show that with arbitrary initial value $\textbf{C}^{(0)}\succeq 0$, $\textbf{C}^{(\ell)}$ converges to
a unique  p.d. matrix.
\begin{mytheorem} \label{guarantee}
The matrix sequence
$\{\textbf{C}^{(\ell)}\}_{l=0,1,\ldots}$ defined by~(\ref{CovFunc5}) converges to a unique positive definite matrix
for any initial covariance matrix $\textbf{C}^{(0)}\succeq \mathbf 0$.
\end{mytheorem}
\begin{proof}
With arbitrary initial value $\textbf{C}^{(0)}\succeq \textbf 0$, following P \ref{P_FUN}.3, we have
$\textbf{U} \succeq  \textbf{C}^{(1)}
\succeq \textbf{L}\succ \textbf{0}$.
On the other hand, according to Theorem \ref{unique},  (\ref{CovFunc5}) has a unique fixed point
$\textbf{C}^{\ast}\succ \textbf{0} $.
Notice that we can always choose a scalar
$\alpha > 1$ such that
\begin{equation}\label{inequality1}
 \alpha \textbf{C}^{\ast}
\succeq
\textbf{C}^{(1)}
\succeq
 \textbf{L}.
\end{equation}
Applying  $\mathcal F(\cdot)$ to (\ref{inequality1})
$l$ times, and using P \ref{P_FUN}.1, we have
\begin{equation}\label{inequality2}
\mathcal F^{l}(\alpha \textbf{C}^{\ast})
\succeq
\mathcal F^{l+1}(\textbf{C}^{(0)})
\succeq
\mathcal F^{l}( \textbf{L}),
\end{equation}
where
$\mathcal F^{l}(\textbf{X})$ denotes applying
$\mathcal F$ on $\textbf{X}$ for $l$ times.


We start from the left inequality in (\ref{inequality2}).
Following the fixed point definition, $\alpha \textbf C^{\ast}=\alpha \mathcal F(\textbf C^{\ast})$.
Then, according to
P \ref{P_FUN}.2, $\alpha \textbf C^{\ast} \succ \mathcal F(\alpha \textbf C^{\ast})$.
Applying $\mathcal F$ again gives $\mathcal F(\alpha \textbf C^{\ast}) \succ \mathcal F^2(\alpha \textbf C^{\ast})$.
Applying $\mathcal F(\cdot)$ repeatedly, we can obtain $\mathcal F^2(\alpha \textbf{C}^{\ast})\\
\succ \mathcal F^3(\alpha \textbf{C}^{\ast})\succ \mathcal F^4(\alpha \textbf{C}^{\ast})$, etc.
 Thus $\mathcal F^l (\alpha \textbf C^{\ast})$ is a decreasing sequence with respect to the partial order induced by the cone of p.s.d. matrices as $l$ increases.
Furthermore, since $\mathcal F(\cdot)$ is bounded below by $\textbf L$, $\mathcal F^l(\alpha  \textbf C^{\ast})$ is convergent.
Finally, since there exists only one fixed point for $\mathcal F(\cdot)$, $\lim_{l\to \infty} \mathcal F^l (\alpha \textbf C^{\ast}) = \textbf  C^{\ast}$.
On the other hand, for the right hand side of (\ref{inequality2}),
as $\mathcal F(\cdot)\succeq \textbf L$,
we have $\mathcal F(\textbf L)\succeq \textbf L$.
Applying $\mathcal F$ repeatedly gives $\mathcal F^2(\textbf L) \succeq \mathcal F(\textbf L)$,
$\mathcal F^3(\textbf L) \succeq \mathcal F^2(\textbf L)$, etc.
So, $\mathcal F^l (\textbf L)$ is an increasing sequence (with respect to the partial order induced by the cone of p.s.d. matrices).
Since $\mathcal F(\cdot)$ is upper bounded by $\textbf U$, $\mathcal F^l(\textbf L)$ is a convergent sequence.
Again due to the unique fixed point, we have $ \lim_{l\to \infty} \mathcal F^l(\textbf L) =\textbf C^{\ast}$.
Finally, taking the limit with respect to $l$ on (\ref{inequality2}) we have
$\lim_{l\to \infty} \mathcal F^{l}(\textbf{C}^{(0)}) = \textbf{C}^{\ast},
$
for arbitrary initial $\textbf{C}^{(0)}\succeq \mathbf 0$.
\end{proof}

According to  Theorem \ref{guarantee}, the covariance matrix $\textbf{C}^{(\ell)}_{f_n\to i}$ converges    if all initial information matrices are
 p.s.d., i.e.,  $\big[\textbf{C}^{(0)}_{f_n\to i}\big]^{-1}\succeq \textbf{0}$
 for all $i \in \mathcal V$ and $f_n \in \mathcal B(i)$.
{Notice that, for the pairwise model, the information matrix does not necessarily converge for all initial non-negative value (in the scalar variable case) as shown in \cite{WalkSum1,minsum09}.}
Moreover, due to the computation of $[\textbf{C}^{(\ell)}_{f_n\to i}]^{-1}$  being independent of the local observations $\textbf{y}_n$,
as long as the network topology does not change, the converged value  $[\textbf{C}^{\ast}_{f_n\to i}]^{-1}$ can be precomputed offline and stored at each node, and there is no need to re-compute $[\textbf{C}^{\ast}_{f_n\to i}]^{-1}$ even if $\textbf{y}_n$ varies.

 Another fundamental question is how fast the convergence is, and this is the focus of the discussion below.
 Since the convergence of  a dynamic system is often studied with the part metric \cite{PartBook},
in the following, we
start by introducing the part metric.

\begin{mydef}\label{mydef}
Part (Birkhoff) Metric\cite{PartBook}:
For arbitrary matrices $\textbf{X}$ and $\textbf{Y}$ with the same dimension,
if there exists
$\alpha\geq 1$ such that $\alpha\textbf{X} \succeq \textbf{Y} \succeq \alpha^{-1} \textbf{X} $,
$\textbf{X} $ and $\textbf{Y}$
 are called the parts,
and
$ \mathrm{d} (\textbf{X}, \textbf{Y})\triangleq
\inf \{\log \alpha: \alpha\textbf{X} \succeq \textbf{Y}\succeq \alpha^{-1} \textbf{X}, \alpha \geq 1\}$
defines a metric  called the part metric.
\end{mydef}

Next, we will show that
$\{\textbf{C}^{(\ell)}\}_{l=1,..}$
converges at a geometric rate  with respect to the part metric in $\mathcal C$, which is constructed as
$$\mathcal C =\{\textbf{C}^{(\ell)}|  \textbf{U} \succeq \textbf{C}^{(\ell)} \succeq \textbf{C}^{\ast}+ \epsilon \textbf{I}\}
\cup
\{\textbf{C}^{(\ell)}|
\textbf{C}^{\ast}- \epsilon \textbf{I} \succeq \textbf{C}^{(\ell)} \succeq \textbf{L}\},$$
where  $\epsilon>0 $ is a scalar and can be arbitrarily small.
\begin{mytheorem}\label{RateCov}
With the initial covariance matrix set to be an arbitrary p.s.d. matrix, i.e., $[\textbf{C}^{(0)}_{f_n\to i}]^{-1}\succeq \textbf{0}$,
the sequence $\{\textbf{C}^{(\ell)}\}_{l=0,1,\ldots}$ converges at a geometric rate  with respect to the part metric in $\mathcal C$.
\end{mytheorem}
\begin{proof}
{Consider two matrices $\textbf{C}^{(\ell)} \in \mathcal C$, and $\textbf{C}^{\ast} \not\in \mathcal C$,
according to Definition \ref{mydef},
we have
$ \mathrm{d} (\textbf{C}^{(\ell)}, \textbf{C}^{\ast})\triangleq
\inf \{\log\alpha: \alpha\textbf{C}^{(\ell)} \succeq \textbf{C}^{\ast}\succeq \alpha^{-1} \textbf{C}^{(\ell)}\}$.
Since $\mathrm{d} (\textbf{C}^{(\ell)}, \textbf{C}^{\ast})$ is the smallest number satisfying $\alpha \textbf{C}^{(\ell)} \succeq \textbf{C}^{\ast} \succeq \alpha^{-1} \textbf{C}^{(\ell)}$, this is equivalent to
\begin{equation}\label{33}
\exp\{\mathrm{d} (\textbf{C}^{(\ell)}, \textbf{C}^{\ast})\}\textbf{C}^{(\ell)}
\succeq
\textbf{C}^{\ast}
\succeq
\exp\{-\mathrm{d} (\textbf{C}^{(\ell)}, \textbf{C}^{\ast})\}
\textbf{C}^{(\ell)}.
\end{equation}
Applying P \ref{P_FUN}.1 to (\ref{33}), we have
$\exp\{\mathrm{d} (\textbf{C}^{(\ell)}, \textbf{C}^{\ast})\}
\mathcal F(\textbf{C}^{(\ell)}
\succeq
\mathcal F(\textbf{C}^{\ast})
\succeq
\exp\{-\mathrm{d} (\textbf{C}^{(\ell)}, \textbf{C}^{\ast})\}
\mathcal F(\textbf{C}^{(\ell)})
$.
Then applying P \ref{P_FUN}.2 and
considering  that   $\exp\{\mathrm{d} (\textbf{C}^{(\ell)}, \textbf{C}^{\ast})\}>1$ and
$\exp\{-\mathrm{d} (\textbf{C}^{(\ell)}, \textbf{C}^{\ast})\}<1$, we obtain
\begin{equation}
\begin{split}
&\exp\{\mathrm{d} (\textbf{C}^{(\ell)}, \textbf{C}^{\ast})\}
\mathcal F(\textbf{C}^{(\ell)})
\\
&\succ
\mathcal F(\textbf{C}^{\ast})
\succ
\exp\{-\mathrm{d} (\textbf{C}^{(\ell)}, \textbf{C}^{\ast})\}
\mathcal F(\textbf{C}^{(\ell)}).
\end{split}
\end{equation}
Notice that, for arbitrary p.d. matrices $\textbf{X}$ and $\textbf{Y}$, if $\textbf{X}-k\textbf{Y}\succ \textbf{0}$ then, by definition that, we have $\textbf{x}^T\textbf{X}\textbf{x}
-k\textbf{x}^T\textbf{Y}\textbf{x}
> {0}$.
Then there must exist $o>0$ that is small enough such that
$\textbf{x}^T\textbf{X}\textbf{x}
-(k+o)
\textbf{x}^T\textbf{Y}\textbf{x}
> {0}$
or equivalently
$\textbf{X}
\succ (k+o)
\textbf{Y}$.
Thus, as $\exp{(\cdot)}$ is a continuous function, there must exist some $\triangle\mathrm{d}>0$ such that
\begin{equation}\label{35}
\begin{split}
&\exp\{-\triangle\mathrm{d}+\mathrm{d} (\textbf{C}^{(\ell)}, \textbf{C}^{\ast})\}
\mathcal F(\textbf{C}^{(\ell)})\\
&\succ
\mathcal F(\textbf{C}^{\ast})
\succ
\exp\{\triangle\mathrm{d}-
\mathrm{d} (\textbf{C}^{(\ell)}, \textbf{C}^{\ast})\}
\mathcal F(\textbf{C}^{(\ell)}).
\end{split}
\end{equation}
Now, using  the definition of part metric, (\ref{35}) is equivalent to
\begin{equation}
-\triangle\mathrm{d}+\mathrm{d}
(\textbf{C}^{(\ell)}, \textbf{C}^{\ast})
\geq
\mathrm{d} (\mathcal F(\textbf{C}^{(\ell)}), \mathcal F(\textbf{C}^{\ast})).
\end{equation}
Hence, we obtain
$\mathrm{d} (\mathcal F(\textbf{C}^{(\ell)}), \mathcal F(\textbf{C}^{\ast}))
<
\mathrm{d}
(\textbf{C}^{(\ell)}, \textbf{C}^{\ast})$.
}
This~result holds for any $\textbf{C}^{(\ell)} \in \mathcal C$,
$\mathrm{d} (\mathcal{F}(\textbf{C}^{(\ell)}), \mathcal{F}(\textbf{C}^{\ast})) <
c
\mathrm{d} (\textbf{C}^{(\ell)},\textbf{C}^{\ast}) $,
where $c=\sup_{\textbf{C}^{l}\in \mathcal{C}} \frac{\mathrm{d}(\mathcal{F}(\textbf{C}^{l}) , \mathcal{F}(\textbf{C}^{\ast})) }{\mathrm{d} (\textbf{C}^{l},\textbf{C}^{\ast})}<1$.
Consequently,
we have
$\mathrm{d} (\textbf{C}^{(\ell)}, \textbf{C}^{\ast}) <
c^{l}
\mathrm{d} (\textbf{C}^{(0)}, \textbf{C}^{\ast}).
$
Thus the sequence $\{\textbf{C}^{(\ell)}\}_{l=1,\ldots}$ converges
at a geometric rate  with respect to the part metric.
\end{proof}

It is useful to have an estimate of the convergence rate of $\textbf{C}^{(\ell)}$ in terms of the more standard induced
matrix norms.
According to  \cite[Lemma 2.3]{matrixnorm},
the convergence rate of $||\textbf{C}^{(0)}- \textbf{C}^{\ast}||$ is dominated by that of $\mathrm{d} (\textbf{C}^{(0)}, \textbf{C}^{\ast})$,
where $||\cdot||$ is a monotone norm defined on the p.s.d. cone, with
$||\cdot||_2$
and $|| \cdot ||_F$ being examples of such matrix norms \cite[2.2-10]{normbook}.
More specifically,
\begin{equation}
\begin{split}
&(2\exp\{\mathrm{d} (\textbf{C}^{(\ell)}, \textbf{C}^{\ast})\} - \exp\{-\mathrm{d} (\textbf{C}^{(\ell)}, \textbf{C}^{\ast})\} - 1)\\
&\quad \quad \quad \times
\min\{||\textbf{C}^{(\ell)}||, ||\textbf{C}^{\ast}||\}
\geq
||\textbf{C}^{(\ell)} - \textbf{C}^{\ast}||.
\end{split}
\end{equation}

The physical meaning of Theorem \ref{RateCov} is that the sequence $\{\textbf{C}^{(\ell)}\}_{l=1,...}$ converges at a geometric rate (the distance between $\textbf{C}^{(\ell)}$ and $\textbf{C}^{\ast}$ decreases exponentially) before $\textbf{C}^{(\ell)}$ enters $\textbf{C}^{\ast}$'s neighborhood, which can be chosen arbitrarily small.
\\[-8mm]

\section{Conclusion}\label{conclusion}
This paper has established the convergence of the exchanged message information matrix of Gaussian belief propagation (BP) for distributed estimation.
We have shown  analytically that, with arbitrary positive semidefinite initial value, the information matrix converges to a unique positive definite matrix at geometric rate.
The convergence guaranteed property and fast convergence rate of the message information matrix pave the way for the convergence analysis of the Gaussian BP message mean vector.

\vfill\pagebreak


\end{document}